\documentclass{article}

\usepackage{amsthm}
\usepackage{times}
\usepackage{graphicx} 
\usepackage{subfigure}
\usepackage{booktabs,multirow}

\usepackage{natbib}

\usepackage{algorithm}
\usepackage{algorithmic}

\usepackage{hyperref}

 \usepackage[accepted]{icml2011}

\icmltitlerunning{Stochastic Low-Rank Kernel Learning for Regression}

\usepackage{amsmath,amsfonts}
\usepackage{array}
\usepackage{xspace}
\usepackage{graphicx}
\usepackage{color}
\usepackage{algorithmic, algorithm,verbatim}
\usepackage{enumerate}

\newtheorem{proposition}{Proposition}
\newtheorem{theorem}{Theorem}
\newtheorem{lemma}{Lemma}

\DeclareMathOperator*{\argmin}{argmin}

\newcommand{\partialmu}[2]{\tfrac{\partial {#1}}{\partial \mu_{#2}}}
\newcommand{\partiallmu}[2]{\tfrac{\partial^2 {#1}}{\partial \mu_{#2}^2}}

\newcommand{\ppartialmu}[3]{\tfrac{\partial^{#3} {#1}}{\partial \mu_{#2}^{#3}}}

\newcommand{\old}{\text{old}}
\newcommand{\new}{\text{new}}

\def\Gnew{G_{\text{new}}}
\def\Gold{G_{\text{old}}}

\def\setS{\mathcal{S}}
\def\rkhs{\mathcal{H}}

\def\realset{\mathbb{R}}

\def\identity{I}

\def\bfc{\boldsymbol{c}}
\def\bfx{\boldsymbol{x}}
\def\bfe{\boldsymbol{e}}

\def\bfg{\boldsymbol{g}}

\def\bfb{\boldsymbol{b}}

\def\bfw{\boldsymbol{w}}

\def\bfm{\boldsymbol{m}}

\def\bfv{\boldsymbol{v}}
\def\bfy{\boldsymbol{y}}

\def\bfalpha{\boldsymbol{\alpha}}
\def\bfmu{\boldsymbol{\mu}}

\def\KtlmInv{\tilde{K}_{\lambda,\bfmu}^{-1} }

\def\SLRKL{{\footnotesize SLKL}}

\usepackage{exscale}
 

\makeatletter
\newenvironment{equationsize*}[1]{%
  \skip@=\baselineskip 
  #1%
  \baselineskip=\skip@ 
  \equation
}{\nonumber\endequation \ignorespacesafterend} 
\makeatother

 

\makeatletter
\newenvironment{alignsize*}[1]{%
  \skip@=\baselineskip 
  #1%
  \baselineskip=\skip@ 
  \start@align\@ne\st@rredtrue\m@ne
}{\endalign\ignorespacesafterend} 
\makeatother

\begin{document}
\twocolumn[
\icmltitle{Stochastic Low-Rank Kernel Learning for Regression}

\icmlauthor{Pierre Machart}{pierre.machart@lif.univ-mrs.fr}
\icmladdress{LIF, LSIS, CNRS, Aix-Marseille Universit\'e}
\icmlauthor{Thomas Peel}{thomas.peel@lif.univ-mrs.fr}
\icmladdress{LIF, LATP, CNRS, Aix-Marseille Universit\'e}
\icmlauthor{Sandrine Anthoine}{anthoine@cmi.univ-mrs.fr}
\icmladdress{LATP, CNRS, Aix-Marseille Universit\'e}
\icmlauthor{Liva Ralaivola}{liva.ralaivola@lif.univ-mrs.fr}
\icmladdress{LIF, CNRS, Aix-Marseille Universit\'e}
\icmlauthor{Herv\'e Glotin}{glotin@univ-tln.fr}
\icmladdress{LSIS, CNRS, Universit\'e du Sud-Toulon-Var}

\icmlkeywords{Stochastic Optimization, Kernel Learning, Low-Rank Approximation}

\vskip 0.3in
]

\begin{abstract} 
We present a novel approach to learn a kernel-based
regression function. It is based on the use of conical
combinations of data-based parameterized kernels and on a new
stochastic convex optimization procedure of which we establish
convergence guarantees. The overall learning procedure has the nice
properties that a) the learned conical combination is automatically
designed to perform the regression task at hand and b) the updates
implicated by the optimization procedure are quite inexpensive.
In order to shed light on the appositeness of our learning strategy,
we present empirical results from experiments conducted on
various benchmark datasets.
\end{abstract}

\section{Introduction}
\label{sec:introduction}
Our goal is to learn a kernel-based regression function, 
tackling at once two problems that commonly arise with kernel methods: 
working with a kernel tailored to the task at hand {\em and}
 efficiently handling problems whose size prevents the Gram matrix from being stored in memory.
Though the present work focuses on regression, 
the material presented here might as well apply to classification. 

Compared with similar methods, we introduce two novelties. 
Firstly, we build conical combinations of rank-1 Nystr\"om approximations,
 whose weights are chosen so as to serve the regression task 
-- this makes our approach different from~\cite{Kumar09} and~\cite{Suykens02}, which focus 
on approximating the full Gram matrix with no concern for any 
specific learning task.
Secondly, to solve the
convex optimization problem entailed by our modeling choice, we provide
an original stochastic optimization procedure based on~\cite{nesterov10}. It has the
following characteristics: i) the computations of the updates are
inexpensive (thanks to the designing choice of using rank-1 approximations)
 and ii) the convergence is guaranteed. To assess
the practicality and effectiveness of our learning procedure, we
conduct a few experiments on benchmark datasets, which allow us to draw positive conclusions on the relevance of our approach.

The paper is organized as follows. Section~\ref{sec:approach} introduces some notation and our learning setting; 
in particular the optimization problem we are interested
in and the rank-1 parametrization of the kernel our approach
builds upon. Section~\ref{sec:solving} describes our new stochastic
optimization procedure, establishes guarantees of convergence and 
 details the computations to be implemented.
Section~\ref{sec:analysis} discusses the hyperparameters inherent to 
 our modeling as well as 
the complexity of the proposed algorithm. 
Section~\ref{sec:simulations} reports results from 
numerical simulations on benchmark datasets.


\section{Proposed Model}
\label{sec:approach}

{\bf Notation }
$\mathcal{X}$ is the input space,
$k : \mathcal{X} \times \mathcal{X} \to
\mathbb{R}$ denotes the (positive) kernel function we have at hand and
$\phi:\mathcal{X}\to\rkhs$ refers to the mapping
$\phi(\boldsymbol{x}):=k(\boldsymbol{x}, \cdot)$ from 
$\mathcal{X}$ to the reproducing kernel Hilbert space $\rkhs$
associated with $k$. Hence,
$k(\bfx,\bfx')\!\!=\!\!\langle\phi(\bfx),\phi(\bfx')\rangle$, with
$\langle\cdot,\cdot\rangle$ the inner product of $\rkhs$.

The training set is $\mathcal{L} :=\{
(\boldsymbol{x}_i,y_i)\}_{i=1}^n\in (\mathcal{X} \times
\mathbb{R})^n $, where $y_i$ is the target
value associated to $\bfx_i$. 
$K =(k(\boldsymbol{x}_i,\boldsymbol{x}_j))_{1\leq i,j\leq
n}\in\mathbb{R}^{n\times n}$ is the Gram matrix of $k$ with respect to
$\mathcal{L}$.
For $m=1,\ldots,n$, $\bfc_m\in\realset^n$ is defined as:
$$\bfc_m := \frac{1}{\sqrt{k(\bfx_m,\bfx_m)}}[k(\bfx_1,\bfx_m),\ldots,k(\bfx_n,\bfx_m)]^{\top}.$$

\subsection{Data-parameterized Kernels}

For $m=1,\dots,n$, $\tilde{\phi}_m : \mathcal{X} \to \tilde{\rkhs}_m$ is the mapping: 
\begin{align}
 \label{phim}
 \tilde{\phi}_m(\bfx) &:= \frac{\langle
   \phi(\bfx),\phi(\bfx_m)\rangle}{k(\bfx_m,\bfx_m)}
 \phi(\bfx_m).
\end{align}
It directly follows that
$\tilde{k}_m$ defined as, $\forall\bfx,\bfx'\in\mathcal{X},$
\begin{align*}
\tilde{k}_m&(\bfx,\bfx'):=\langle\tilde{\phi}_m(\bfx),\tilde{\phi}_m(\bfx')\rangle
=\frac{k(\bfx,\bfx_m)k(\bfx',\bfx_m)}{k(\bfx_m,\bfx_m)},
\end{align*}
is indeed a positive kernel.
Therefore, these parameterized kernels $\tilde{k}_m$ give rise to a family $(\tilde{K}_m)_{1\leq
  m\leq n}$ of Gram matrices of the following form:
\begin{equation}
\label{rk1nyst}
 \tilde{K}_m =(\tilde{k}_m(\bfx_i,\bfx_j))_{1\leq i,j\leq n}= \bfc_m \bfc_m^T,
\end{equation}
which can be seen as rank-1 Nystr\"om approximations of the full Gram matrix $K$~\cite{Drineas05,Williams01Nystrom}.

As studied in~\cite{Kumar09}, it is sensible to consider convex 
combinations of the $\tilde{K}_m$ if they are of very low
rank.  Building on this idea,  we will 
investigate the use of a parameterized
Gram matrix of the form:
\begin{equation}
 \label{ensnyst}
 \tilde{K}(\bfmu) = \sum_{m \in \mathcal{S}} \mu_m \tilde{K}_m \quad \text{with} \quad \mu_m \geq 0,
\end{equation}
where $\mathcal{S}$ is a set of indices corresponding to the specific rank-one approximations used.
Note that since we consider {\em conical} combinations of the
$\tilde{K}_m$, which are all positive
semi-definite, $\tilde{K}(\bfmu)$ is positive semi-definite as well. 

Using~\eqref{phim}, one can show that the kernel $\tilde{k}_{\bfmu}$, associated 
to our parametrized Gram matrix $\tilde{K}(\bfmu)$, is such that:
\begin{align}
 \label{ktilde}
 \tilde{k}_{\bfmu}(\bfx, \bfx') &= \langle\phi(\bfx),\phi(\bfx')\rangle_A =\phi(\bfx)^{\top}A\phi(\bfx),\\
 \text{with}\quad A:&= \sum_{m \in \mathcal{S}} \mu_m\frac{\phi(\bfx_m) \phi(\bfx_m)^{\top}}{k(\bfx_m,\bfx_m)}. 
\end{align}
In other words, our parametrization 
 induces a modified metric
in the feature space ${\cal H}$ associated to $k$.
On a side note, remark that when $\mathcal{S} = \{1\ldots,n\}$ (i.e. all the columns are picked) and we have 
uniform weights $\bfmu$, then $\tilde{K}(\bfmu) = K K^{\top}$, which
is a matrix encountered when working with the
so-called empirical kernel map~\cite{Scholkopf99}.

From now on, $M$ denotes the size of $\mathcal{S}$ and $m_0$ refers to the number of
non-zero components of $\bfmu$ (i.e. it is the 0-pseudo-norm of $\bfmu$).

\subsection{Kernel Ridge Regression}
Kernel Ridge regression (KRR) is the kernelized version of the popular
ridge regression~\cite{Hoerl70} method. The associated optimization problem reads:
\begin{equation}
\label{KRRprim}
 \min_{\boldsymbol{w}} \left\{ \lambda \|\bfw\|^2 + \sum_{i=1}^n{ \left( y_i - \langle \bfw,\phi(\bfx_i) \rangle \right) ^2} \right\},
\end{equation}
where $\lambda>0$ is a regularization parameter.

Using $\identity$ for the identity matrix, the following dual formulation may be considered:
\begin{equation}
 \label{KRRdual}
 \max_{\bfalpha\in\realset^n} \left\{F_{KRR}(\boldsymbol{\alpha}) := \boldsymbol{y}^T \boldsymbol{\alpha} - \frac{1}{4 \lambda} \boldsymbol{\alpha}^T (\lambda \identity + K) \boldsymbol{\alpha}\right\}.
\end{equation}
The solution $\bfalpha^*$ of the concave problem~\eqref{KRRdual} and the optimal solution
$\bfw^*$ of \eqref{KRRprim} are connected through the equality  
$$\bfw^*=\frac{1}{2 \lambda}\sum_{i=1}^n\alpha^*_i\phi(\bfx_i),$$
and $\boldsymbol{\alpha}^*$ can be found by setting the gradient of $F_{KRR}$ to zero, to give
\begin{equation}
 \label{alphaKRR}
 \boldsymbol{\alpha}^* = 2  ( \identity + \tfrac{1}{\lambda}K)^{-1} \boldsymbol{y}.
\end{equation}
The value of the objective function at $\bfalpha^*$ is then:
\begin{equation}
\label{objminKRR}
 F_{KRR}(\bfalpha^*) =  \boldsymbol{y}^T ( \identity + \tfrac{1}{\lambda}K)^{-1} \boldsymbol{y},
\end{equation}
and the resulting regression function is given by:
\begin{equation}
 \label{optfunc}
 f(\bfx) = \frac{1}{2 \lambda} \sum_{i=1}^n \alpha_i^* k(\bfx_i,\bfx).
\end{equation}

\subsection{A Convex Optimization Problem}
KRR may be solved by solving the linear system $(\identity +
\tfrac{K}{\lambda})\bfalpha=2\bfy$, at a cost of $O(n^3)$ operations. This
might be prohibitive for large $n$, even more so if the matrix $\identity+ \tfrac{K}{\lambda}$ does not fit into memory. To cope with this possible
problem, we work with $\tilde{K}(\bfmu)$~\eqref{ensnyst} instead of
 the Gram matrix $K$. As we shall see, this not only makes it possible to
avoid memory issues but it also allows us to set
up a learning problem where both $\bfmu$ and a regression function
are sought for at once. This is very similar to the Multiple
Kernel Learning paradigm~\cite{Rakoto08} where one learns an
optimal kernel along with the target function. 

To set up the optimization problem we are interested in, we proceed in
a way similar to~\cite{Rakoto08}. For $m=1,\ldots,n$, define the
Hilbert space  $\tilde{\rkhs}_m'$ as:
\begin{equation}
 \label{KmHS}
 \tilde{\rkhs}_m' := \left\{f \in \tilde{\rkhs}_m \left\vert \frac{\|f\|_{\tilde{\rkhs}_m}}{\mu_m} < \infty\right. \right\}.
\end{equation}

One can prove~\cite{Arons50} that $\tilde{\rkhs} = \bigoplus \tilde{\rkhs}_m'$ is the RKHS associated to $\tilde{k}=\sum\mu_m\tilde{k}_m$.
Mimicking the reasoning of~\cite{Rakoto08}, our primal optimization problem reads:
\begin{align}
\label{NKRRprim}
  \nonumber \min_{\{f_m\}, \bfmu}
  &  \left \{\lambda \sum_{m \in \mathcal{S}} \frac{1}{\mu_m} \|f_m\|_{\tilde{\rkhs}_m'}^2 + \sum_{i=1}^n (y_i - \sum_{m \in \mathcal{S}} f_m(\bfx_i))^2\right\}, \\
  & \text{s.t. } \sum_{m \in \mathcal{S}} \mu_m \leq n_1~,\quad \mu_m \geq 0,
\end{align}
where $n_1$ is a parameter controlling the 1-norm of $\bfmu$.
As this problem is also convex in $\bfmu$, using the earlier results on the KRR problem, \eqref{NKRRprim} is equivalent to:
\begin{align}
 \label{minmax}
  \min_{\bfmu \geq 0} &\left\{\max_{\boldsymbol{\alpha}}\; \boldsymbol{y}^T \boldsymbol{\alpha} - \frac{1}{4 \lambda} \boldsymbol{\alpha}^T (\lambda \identity + \tilde{K}(\bfmu)) \boldsymbol{\alpha}\right\} \nonumber \\
   = \min_{\bfmu \geq 0}&\left\{\boldsymbol{y}^T (\identity + \tfrac{1}{\lambda}\tilde{K}(\bfmu))^{-1} \boldsymbol{y}\right\}
   \text{s.t.} \sum_{m \in \mathcal{S}} \mu_m \leq n_1.
\end{align}

Finally,  using the equivalence between Tikhonov and Ivanov
regularization methods~\cite{Vasin70}, we obtain the {\em convex} and
{\em smooth} optimization problem we focus on: 
\begin{equation}
 \label{ourprob}
 \min_{\bfmu \geq 0}
\left\{F(\bfmu) :=  {\small \boldsymbol{y}^T ( \identity + }\tfrac{1}{\lambda} {\small \tilde{K}(\bfmu))^{-1} \boldsymbol{y} + 
\nu \sum_m \mu_m} \right\}.
\end{equation}

 The regression function $\tilde{f}$ is derived using \eqref{phim}, a minimizer $\bfmu^*$ of the latter problem and
the accompanying weight vector $\boldsymbol{\alpha}^*$ such that
\begin{equation}
\label{eq:alphastar}
\bfalpha^{* } = 2\left(\identity+\tfrac{1}{\lambda}\tilde{K}(\bfmu^{*})\right)^{-1} \bfy,
\end{equation} 
(obtained adapting~\eqref{alphaKRR} to the case $K=K(\bfmu^{*})$).  We have:
 \begin{align}
 \label{optappfunc}
 \nonumber\tilde{f}(\bfx)
 &= \frac{1}{2 \lambda} \sum_{i=1}^n \alpha_i^* \tilde{k}(\bfx_i,\bfx)
=\frac{1}{2 \lambda} \sum_{m \in
   \mathcal{S}} \mu_m^*\sum_{i=1}^n \alpha_i^*  \tilde{k}_m(\bfx_i,\bfx)\\
 &= \frac{1}{2 \lambda} \sum_{m \in \mathcal{S}} \tilde{\alpha}_m^* k(\bfx_m,\bfx),
 \end{align}
\begin{equation}
\label{eq:alphatilde}
\text{where}\qquad\tilde{\alpha}_m^*:=\mu_m^* \frac{\bfc_m^{\top}\bfalpha^*}{\sqrt{k(\bfx_m,\bfx_m)}}.
\end{equation}


\section{Solving the problem}
\label{sec:solving}

We now introduce a new stochastic optimization procedure to
solve~\eqref{ourprob}. It implements a coordinate descent strategy
with step sizes that use second-order information.

\subsection{A Second-Order Stochastic Coordinate Descent}

Problem~(\ref{ourprob}) is a constrained minimization based on the differentiable
 and convex objective function $F$. Usual convex optimization methods
 (such as projected gradient descent, proximal methods)  may be
 employed to solve this problem, but they may be too computationally
 expensive if $n$ is very large, which is essentially due to a
 suboptimal exploitation of the parametrization of the
 problem. Instead, the optimization strategy we propose
 is specifically tailored to take advantage of the parametrization of
 $\tilde{K}(\bfmu)$. 

Algorithm~\ref{algo:scnd} depicts our stochastic descent method, inspired by~\cite{nesterov10}.
At each iteration, a randomly chosen coordinate of $\bfmu$ is updated via a Newton step. This method has two essential features: 
i) using coordinate-wise updates of $\bfmu$ involves only partial derivatives which can be easily computed and ii) the 
stochastic approach ensures a reduced memory cost while still guaranteeing convergence.

\begin{algorithm}[!h]
    \caption{Stochastic Coordinate Newton Descent}
    \label{algo:scnd}
  \begin{algorithmic}
    \STATE {\bfseries Input:} $\bfmu^{0}$ random.
    \REPEAT
    \STATE Choose coordinate $m_k$ uniformly at random in $\setS$.
        \STATE Update :
        $\mu^{k+1}_m = \mu^{k}_m$ if $m\not = m_k$ and
  	\begin{equation}
 	   \!\!\!\!\mu^{k+1}_{m_k} \!=  \!\underset{v\geq 0}{\argmin}\partialmu{F(\bfmu^k)}{m_k}(v - \mu^k_{m_k} ) + 
            \tfrac{1}{2} \partiallmu{F(\bfmu^k)}{m_k}(v - \mu^k_{m_k} )^2,
	  \label{eq:update}
	\end{equation}
    \UNTIL{$F(\bfmu^{k}) - F(\bfmu^{k-M})<\epsilon F(\bfmu^{k-M})$}
  \end{algorithmic}
\end{algorithm}

Notice that the Stochastic Coordinate Newton Descent (SCND) is similar
to the algorithm proposed in~\cite{nesterov10}, except that we replace
the Lipschitz constants by the second-order partial derivatives
$\partiallmu{F(\mu^k)}{m_k}$. Thus, we replace a constant step-size
gradient descent by a the Newton-step in~\eqref{eq:update}, which
allows us to make larger steps.

We show that for the function $F$ in~\eqref{ourprob}, SCND does 
 provably
converge to a minimizer of Problem~\eqref{ourprob}. First, we rewrite~\eqref{eq:update} as a Newton step and compute the partial derivatives:
\begin{proposition} Eq.~\eqref{eq:update} is equivalent to
  \begin{equation}
    \label{eq:newtonstep}
    \mu^{k+1}_{m_k} = \left\{ 
    \begin{array}{l l}
    \left(\mu^{k}_{m_k} - {\partialmu{F(\mu^k)}{m_k} }/{\partiallmu{F(\bfmu^k)}{m_k} }\right)_+\text{if} \; \partiallmu{F(\bfmu^k)}{m_k}\! \not=\! 0\\
    0 \text{ otherwise.}\\
  \end{array} \right.
  \end{equation}
\end{proposition}
\begin{proof}\eqref{eq:newtonstep} gives the optimality conditions for~\eqref{eq:update}.  \end{proof}

\begin{proposition} The partial  derivatives $\ppartialmu{F(\bfmu)}{m}{p}$ are:
  \begin{align}
    \label{eq:ppartial}
    \partialmu{F(\bfmu)}{m} &= -\lambda(\bfy^{\top}\KtlmInv\bfc_m)^2 + \nu,\\
    \ppartialmu{F(\bfmu)}{m}{p} &=(-1)^p p! \lambda(\bfy^{\top}\KtlmInv\bfc_m)^2(\bfc_m^{\top}\KtlmInv\bfc_m)^{p-1},\nonumber \\
& \text{with}\ p \geq 2 \text{ and } \KtlmInv := (\lambda I + \tilde{K}(\bfmu))^{-1}\label{eq:KtlmInv}.
  \end{align}
 \end{proposition}
\begin{proof}Easy but tedious calculations give the results.\end{proof}
\begin{theorem}[Convergence]
  \label{th:cv}
  For any sequence  $\{m_k\}_k$, the sequence $\{F(\bfmu^k)\}_k$ verifies: 
  \begin{enumerate}[(a)]\itemsep0pt
   \item $ \forall k,\,F(\bfmu^{k+1})\leq F(\bfmu^k).$
   \item $ \lim_{k\to\infty} F(\bfmu^k) = \min_{\bfmu\geq0} F(\bfmu).$
  \end{enumerate} 

 Moreover, if there exists a minimizer $\bfmu^*$ of $F$ such that the
 Hessian $\nabla^2F(\bfmu^*)$ is positive definite then:
  \begin{enumerate}[(a)]\itemsep0pt
    \setcounter{enumi}{2}
    \item $\bfmu^*$ is the unique minimizer of  $F$.
     The sequence $\{\bfmu^k\}$ converges to   $\bfmu^*$: $||\bfmu^k\!-\!\bfmu^*||{\to}0$.
  \end{enumerate} 
\end{theorem}

\begin{proof}[Sketch of proof]
  \begin{enumerate}[(a)]
    \item Using that $ \ppartialmu{F(\bfmu)}{m}{3}\leq 0$ (see ~\eqref{eq:ppartial}), one shows that the Taylor series truncated to the second order: 
$\bfv \to F(\bfmu) + \partialmu{F(\bfmu)}{m}(\bfv_{m} - \bfmu_{m} ) + \tfrac{1}{2}\partiallmu{F(\bfmu)}{m}( \bfv_m - \bfmu_{m}  )^2, $
is a quadratic upper-bound of $F$ that matches $F$ and $\nabla F$ at point $\bfmu$ (for any fixed $m$ and $\bfmu$).
From this, the update formula~\eqref{eq:update} yields $F(\bfmu^{k+1})\leq F(\bfmu^{k})$.

     \item First note that $||\bfmu^k||\leq F(\bfmu^0)$ and extract a converging subsequence $\{\bfmu^{\phi(k)}\}$. Denote the limit by $\hat{\bfmu}$. Separating the cases where $ \partiallmu{F(\hat{\bfmu})}{m}$ is zero or not, one shows that  $\hat{\bfmu}$ satisfies the optimality conditions: 
$\langle\nabla F(\hat{\bfmu}),\bfv -\hat{\bfmu} \rangle \geq 0, \,\forall \bfv\geq 0$. Thus $\hat{\bfmu}$ is a minimizer of $F$ and we have $\lim F(\bfmu^k)=\lim F(\bfmu^{\phi(k)}) = F(\hat{\bfmu}) =  \min_{\bfmu\geq0} F(\bfmu)$.

    \item 
   is standard in convex optimization.
  \end{enumerate}
  \vspace*{-6mm}
\end{proof}

\subsection{Iterative Updates}
\label{subsec:updates}
One may notice that the computations of the derivatives~\eqref{eq:ppartial}, as well as the computation of $\boldsymbol{\alpha}^*$, depend on $\KtlmInv$.
Moreover, the dependency in $\bfmu$, for all those quantities, only lies in $\KtlmInv$.
Thus, a special care need be taken on how $\KtlmInv$ is stored and updated throughout.

Let $\mathcal{S}_{\bfmu}^+ = \left\{m \in \mathcal{S} \vert \mu_m >
  0\right\}$ and $m_0=\|\bfmu\|_0=|\mathcal{S}_{\bfmu}^+|$.
Let $C=[\bfc_{i_1}\cdots\bfc_{i_{m_0}}]$ be the concatenation of the $\bfc_{i_j}$'s, for $i_j \in
\mathcal{S}_{\bfmu}^+$ and $D$ the diagonal matrix with diagonal
elements $\mu_{i_j}$, for $i_j \in
\mathcal{S}_{\bfmu}^+$. Remark that throughout the iterations
  the sizes of $C$ and $D$ may vary.
Given~\eqref{eq:KtlmInv} and using Woodbury formula (Theorem~\ref{th:woodbury},
Appendix), we have:
\begin{equation}
 \label{Kinv}
 \begin{aligned}
 \KtlmInv
 =\big(\lambda \identity + C D C^{\top}\big)^{-1} = \frac{1}{\lambda} \identity - \frac{1}{\lambda^2} C G C^{\top}
\end{aligned}
\end{equation}
\begin{equation}
 \label{invM}
\text{with}\qquad G := \Big(D^{-1} + \frac{1}{\lambda} C^{\top} C\Big)^{-1}.
\end{equation}
Note that $G$ is a square matrix of order $m_0$ and that an update on $\bfmu$ will require an update on $G$.
Even though updating $G^{-1}$, i.e. $D^{-1}+ \frac{1}{\lambda} C^{\top} C$, is trivial, it is more efficient to directly
store and update $G$. This is what we describe now.

At each iteration, only one coordinate of $\bfmu$ is updated.
Let $p$ be the index of the updated coordinate, $\bfmu_{\old}$,
$C_{\old}$, $D_{\old}$ and $G_{\old}$, the 
vectors and matrices before the update and $\bfmu_{\new}$, $C_{\new}$,
$D_{\new}$ and $G_{\new}$ the updated matrices/vectors.
Let  also  $\boldsymbol{e}_p$  bethe vector whose $p$th
coordinate is $1$ while other coordinates are $0$. 
We encounter four different cases.
\paragraph{Case 1:  $\mu_p^{\old} = 0$ and $\mu_p^{\new} =
  0$.} No update needed:
\begin{equation}
\label{eq:noupdate}
\Gnew = \Gold.
\end{equation}

\paragraph{Case 2: $\mu_p^{\old} \neq 0$ and $\mu_p^{\new}
  \neq 0$.} Here, $C_{\old} = C_{\new}$
 and
$$D_{\new}^{-1} = D_{\old}^{-1} +
\Delta_p\bfe_p\bfe_p^{\top}, \quad\text{where}\quad\Delta_p:=\frac{1}{\mu_p^{\new}} - \frac{1}{\mu_p^{\old}}.$$
Then, using Woodbury formula, we have:
\begin{equation}
 \label{updneqneq}
 G_{\new}
 = \Big(G_{\old}^{-1} + \Delta_p\bfe_p\bfe_p^{\top}\Big)^{-1}
 = G_{\old} - \frac{\Delta_p}{1 +\Delta_pg_{pp}} \bfg_p\bfg_p^{\top},
\end{equation}
with $g_{pp}$ the $(p,p)$th entry of $G_{\old}$ and
$\bfg_p$ its $p$th column.

\paragraph{Case 3: $\mu_p^{\old} \neq 0$ and $\mu_p^{\new} = 0$.} Here, $\mathcal{S}_{\bfmu_{\new}}^+ = \mathcal{S}_{\bfmu_\old}^+ \setminus \{p\}$.
It follows that we have to remove $\bfc_p$ from $C_{\old}$ to have $C_{\new}$.
To get $G_{\new}$, we may consider the previous
update formula when $\mu_p^{\new}\to 0$ (that is, when $\Delta_p\to +\infty$). Note that we can use the
previous formula because $\mu_p \mapsto
\KtlmInv$ is well-defined and continuous at $0$.
Thus, as
$\lim_{\mu_p^{\new}\to 0}\frac{\Delta_p}{1+\Delta_pg_{pp}} = \frac{1}{g_{pp}} ,$
we have:
\begin{equation}
 \label{updneqeq}
 \begin{aligned}
 G_{\new}
 &= \left(G_{\old}- \frac{1}{g_{pp}}\bfg_p\bfg_p^{\top}\right)_{\setminus\{p\}},
\end{aligned}
\end{equation}
where $A_{\setminus\{p\}}$ denotes the matrix $A$ from which the $p$th column
and $p$th row have been removed. 

\paragraph{Case 4: $\mu_p^{\old} = 0$ and $\mu_p^{\new}
  \neq 0$.} We have $C_{\new} = [C_{\old}\; \bfc_p\big]$.
Using~\eqref{invM}, it follows that
\begin{equation*}
 \begin{aligned}
 G_{\new}&=\left(\begin{array}{cc}D_{\old}^{-1}+\frac{1}{\lambda}C_{\old}^{\top}C_{\old} & \frac{1}{\lambda} C_{\old}^{\top} \bfc_p \\ \frac{1}{\lambda} \bfc_p^{\top} C_{\old} & \frac{1}{\mu_p^{\new}} + \frac{1}{\lambda} \bfc_p^{\top} \bfc_p \end{array}\right)^{-1}
\\
&=\left(\begin{array}{cc} G_{\old}^{-1} & \frac{1}{\lambda} C_{\old}^{\top} \bfc_p \\ \frac{1}{\lambda} \bfc_p^{\top} C_{\old} & \frac{1}{\mu_p^{\new}} + \frac{1}{\lambda} \bfc_p^{\top} \bfc_p \end{array}\right)^{-1}\\
 &=\begin{pmatrix}A & \bfv \\ \bfv^{\top} & s\end{pmatrix},
\end{aligned}
\end{equation*}
where, using the block-matrix inversion formula of
Theorem~\ref{th:inversionadd} (Appendix), we have:
\begin{align}
s &= \left(\frac{1}{\mu_p^{\new}} + \frac{1}{\lambda} \bfc_p^{\top} \bfc_p - \frac{1}{\lambda^2} \bfc_p^{\top} C_{\old} G_{\old}C_{\old}^{\top} \bfc_p\right)^{-1}\notag\\
 \boldsymbol{v} &= -\frac{s}{\lambda} G_{\old}C_{\old}^{\top} \bfc_p \label{updinvM0neq0}
\\
 \boldsymbol{A} &= G_{\old}+ \frac{1}{s} \boldsymbol{v} \boldsymbol{v}^{\top}\notag.
\end{align}

\noindent{\bf Complete learning algorithm.} Algorithm~\ref{alg:slrk}
depicts the full Stochastic Low-Rank
Kernel Learning  algorithm (\SLRKL), which recollects all the pieces just described.

\begin{algorithm}[tb]
    \caption{\SLRKL: Stochastic Low-Rank Kernel Learning }
    \label{alg:slrk}
  \begin{algorithmic}
    \STATE {\bfseries inputs:} ${\cal L}:=\{(\bfx_i,y_i)\}_{i=1}^n$,
    $\nu>0$, $M>0$, $\epsilon>0$.
    \STATE {\bf outputs:} $\bfmu$, $G$ and $C$ (yield $(\lambda\identity+K(\mu))^{-1}$ from~\eqref{Kinv}).
    \STATE ~

    \STATE{\bf initialization:} $\bfmu^{(0)}={\bf 0}.$
    \REPEAT
    \STATE Choose coordinate $m_k$ uniformly at random in $\setS$.
    \STATE Update $\bfmu^{(k)}$ according to~\eqref{eq:newtonstep}, by
    changing only the $m_k$-th coordinate $\mu^{k}_{m_k}$ of
    $\bfmu^{(k)}$:
    \begin{itemize}\itemsep-\baselineskip
    \item compute the second order derivative
      $$h = \lambda(\bfy^{\top}\KtlmInv\bfc_{m_k})^2(\bfc_{m_k}^{\top}\KtlmInv\bfc_{m_k})\ ;$$
    \item {\bf if}
      $h>0$
      {\bf then}
      \begin{equation*}
      \small
      \mu_{m_k}^{(k+1)} = \max \left(0, \mu_{m_k}^{(k)}+\frac{\lambda(\bfy^{\top}\KtlmInv\bfc_{m_k})^2 - \nu}{h}\right);
      \end{equation*}
      {\bf else}
      $\qquad\mu_{m_k}^{(k+1)}= 0.$
    \end{itemize}
    \STATE Update $G^{(k)}$ and $C^{(k)}$ according to~\eqref{eq:noupdate}-\eqref{updinvM0neq0}.
  \UNTIL{$F(\bfmu^{k}) - F(\bfmu^{k-M})<\epsilon F(\bfmu^{k-M})$}
  \end{algorithmic}
\end{algorithm}


\section{Analysis}
\label{sec:analysis}
Here, we discuss the relation between $\lambda$ and $\nu$ and we argue that there is no need to
keep both hyperparameters. In addition, we provide
a short analysis on the runtime complexity of our learning procedure.

\subsection{Pivotal Hyperparameter $\lambda\nu$}
First recall that we are interested in the minimizer $\bfmu_{\lambda,\nu}^*$ of constrained optimization
problem~\eqref{ourprob}, i.e.:
\begin{equation}
\label{eq:mustar}
\bfmu_{\lambda,\nu}^*=\argmin_{\bfmu\geq 0}F_{\lambda,\nu}(\bfmu),
\end{equation}
where, for the sake of clarity, we purposely show the dependence on $\lambda$ and $\nu$ of the
objective function $F_{\lambda,\nu}$
\begin{align}
F_{\lambda, \nu}(\bfmu)
&=\bfy^{\top} \left(\identity
+ \tilde{K}\left(\tfrac{\bfmu}{\lambda}\right)\right)^{-1} \bfy + \lambda\nu
\sum_m \tfrac{\mu_m}{\lambda},\label{eq:flambdanu}
\end{align}
We may name
$\bfalpha_{\lambda,\nu}^*$, $\tilde{\bfalpha}_{\lambda,\nu}^*$ the
weight vectors associated with $\bfmu_{\lambda,\nu}^*$
(see~\eqref{eq:alphastar} and~\eqref{eq:alphatilde}).
We have the following: 
\begin{proposition}
\label{prop:lambdanu}
Let $\lambda, \nu, \lambda', \nu'$ be strictly positive real numbers. If
$\lambda\nu=\lambda'\nu'$ then
$$\bfmu_{\lambda',\nu'}^*=\tfrac{\lambda'}{\lambda}\bfmu_{\lambda,\nu}^{*},
\quad\text{and}\quad\tilde{f}_{\lambda,\nu}=\tilde{f}_{\lambda', \nu'}.$$
As a direct consequence:
$$\forall \lambda,\nu\geq 0,\;
\tilde{f}_{\lambda,\nu}=\tilde{f}_{1,\lambda\nu}.$$
\end{proposition}
\begin{proof}
Suppose that we know $\bfmu_{\lambda,\nu}^*$. Given the
definition~\eqref{eq:flambdanu} of $F_{\lambda,\nu}$ and using $\lambda\nu=\lambda'\nu'$, we 
have 
$$F_{\lambda,\nu}(\bfmu)=F_{\lambda',\nu'}\left(\tfrac{\lambda'}{\lambda}\bfmu\right)$$
Since the only
constraint of problem~\eqref{eq:mustar} is the nonnegativity
 of the components of $\bfmu$, it directly follows that
$\lambda'\bfmu_{\lambda,\nu}^*/\lambda$ is a minimizer of
$F_{\lambda',\nu'}$ (under these constraints), hence
$\bfmu^*_{\lambda',\nu'}=\lambda'\bfmu_{\lambda,\nu}^*/\lambda$.

To show $\tilde{f}_{\lambda,\nu}=\tilde{f}_{\lambda',\nu'}$, it
suffices to observe that, according to the way $\bfalpha_{\lambda,\nu}^*$ is defined
(cf. ~\eqref{eq:alphastar}),
\begin{align*}
\bfalpha_{\lambda',\nu'}^*&=2\left(\identity+K\left(\tfrac{\bfmu_{\lambda',\nu'}^*}{\lambda'}\right)\right)^{-1}\bfy\\
&=2\left(\identity+K\left(\tfrac{\lambda'}{\lambda}\tfrac{\bfmu_{\lambda,\nu}^*}{\lambda'}\right)\right)^{-1}\bfy=\bfalpha_{\lambda,\nu}^*,
\end{align*}
and, thus,
$\tilde{\bfalpha}_{\lambda',\nu'}^*=\lambda'\tilde{\bfalpha}_{\lambda,\nu}^*/\lambda.$
The definition~\eqref{optappfunc} of $\tilde{f}_{\lambda,\nu}$ then gives
$\tilde{f}_{\lambda,\nu}=\tilde{f}_{\lambda',\nu'}$, which entails $\tilde{f}_{\lambda,\nu}=\tilde{f}_{1,\lambda\nu}$.
\end{proof}
This proposition has two nice consequences. First, it says that the
pivotal hyperparameter is actually the product $\lambda\nu$: this {\em
  is} the quantity that parametrizes the learning problem (not
$\lambda$ or $\nu$, seen independently).
Thus, the set of regression functions, defined by the $\lambda$ and $\nu$ 
hyperparameter space, can be described by exploring the set of vectors 
$(\bfmu_{1,\nu}^*)_{\nu>0}$, which only depends on a single parameter.
Second, considering  $(\bfmu_{1,\nu}^*)_{\nu>
0}$ allows us to work with the family of objective functions $(F_{1,\nu})_{\nu>0}$, which are 
well-conditioned numerically as the hyperparameter $\lambda$ is set to $1$.

\subsection{Runtime Complexity and Memory Usage}
\label{complexanalysis}
For the present analysis, let us assume that we pre-compute the
$M$ (randomly) selected columns $\bfc_1,\ldots,\bfc_M$.
If $a$ is the cost of computing a column $\bfc_m$, the
pre-computation has a cost of $O(Ma)$ and has a memory usage of $O(n
M)$.

At each iteration, we have to compute the first and second-order
derivatives of the objective function, as well as its value and the
weight vector $\bfalpha$.
Using \eqref{Kinv}, \eqref{eq:ppartial}, \eqref{ourprob} and \eqref{eq:alphastar},
one can show that those operations have a complexity of $O(n m_0)$ if $m_0$ is
the zero-norm of $\bfmu$.

Besides, in addition to $C$, we need to store $G$ for a memory cost of $O(m_0^2)$.
Overall, if we denote the number of iterations by $k$, the algorithm has a memory 
cost of $O(n M + m_0^2)$ and a complexity of $O(k n m_0 + M a)$.

If memory is a critical issue, one may prefer to compute the columns
$\bfc_m$ on-the-fly and
$m_0$ columns need to be stored instead of $M$ (this might be a
substantial saving in terms of memory as can be seen in the next section).
This improvement in term of memory usage implies an additive cost in
the runtime complexity.
In the worst case, we have to compute a new column $\bfc$ at each iteration.
The resulting memory requirement scales as $O(n m_0 + m_0^2)$ and
the runtime complexity varies as $O(k (n m_0 + a))$.

\section{Numerical Simulations}
\label{sec:simulations}

We now present results from various numerical experiments, for which
we describe the datasets and the protocol  used.
We study the influence of the
different parameters of our learning approach on the results and  compare the performance of our algorithm to that of related methods.

\subsection{Setup}
First, we use a toy dataset (denoted by \emph{sinc}) to better understand the role and influence of the parameters.
It consists in regressing the cardinal sine of the two-norm
(i.e. $\bfx\mapsto\sin(\|\bfx\|)/\|\bfx\|$) of
random two-dimensional points, each drawn uniformly between $-5$ and $+5$.
In order to have a better idea on how the solutions may or may not
over-fit the training data, we add some white Gaussian noise on the
target variable of the randomly picked 1000 training points (with a 10 dB signal-to-noise ratio). The test set is
made of 1000 non-noisy independent instance/target pairs.

We then assess our method on two UCI datasets: Abalone
(\emph{abalone}) and Boston Housing (\emph{boston}), using
the same normalizations, Gaussian kernel parameters ($\sigma$ denotes the kernel width) and data partition
as in~\cite{Smola00}.  The United States Postal Service (\emph{USPS})
dataset is used with
the same setting as in \cite{Williams01Nystrom}. Finally, the Modified National Institute of Standards and Technology (\emph{MNIST}) dataset 
is used with the same pre-processing as in~\cite{Maji09}.
Table~\ref{datasets} summarizes the characteristics of all the datasets we used.

\begin{table}[!h]
\caption{Datasets used for the experiments.\label{datasets}}
\centering
\begin{tabular}{ccccc}
\toprule 
{\bf dataset} & {\bf \#features} & {\bf \#train ($n$)} & {\bf \#test} & $\boldsymbol{\sigma}^2$\\
\midrule
\emph{sinc} & 2 & 1000 & 1000 & 1\\
\emph{abalone} & 10 & 3000 & 1177 & $2.5$\\
\emph{boston} & 13 & 350 & 156 & $3.25$\\
\emph{USPS} & 256 & 7291 & 2007 & 64\\
\emph{MNIST} & 2172 & 60000 & 10000 & $4$\\
\bottomrule
\end{tabular}
\end{table}

As displayed in Algorithm~\ref{algo:scnd}, at each iteration $k > M$, we check if $F(\bfmu^{k}) - F(\bfmu^{k-M})<\epsilon F(\bfmu^{k-M})$ holds.
If so, we stop the optimization process.
$\epsilon$ thus controls our stopping criterion.
In the experiments, we set $\epsilon = 10^{-4}$ unless otherwise stated and we set $\lambda$ to $1$ 
for all the experiments and we run simulations for various values of $\nu$ and $M$.
In order to assess the variability incurred by the stochastic nature
of our learning algorithm, we run each
experiment 20 times.

\subsection{Influence of the parameters}
\subsubsection{Evolution of the objective}
\begin{figure}[tb]
\centering
 \includegraphics[scale=0.18]{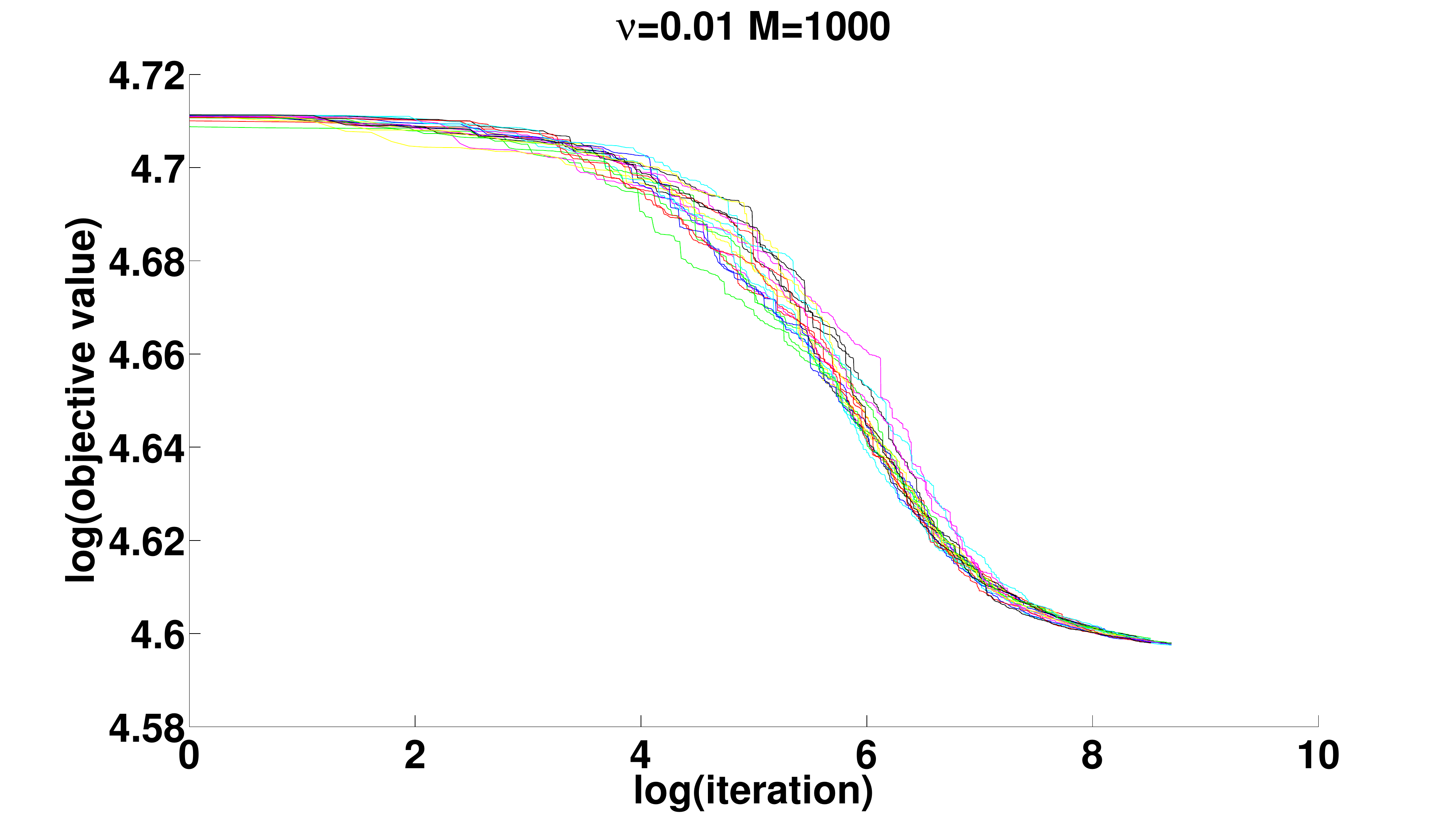}
\caption{Evolution of the objective during the optimization process for the \emph{sinc} dataset with $\nu = 0.01$, $M = 1000$ (for 20 runs).}
\label{objdec}
\end{figure}
We have established (Section~\ref{sec:solving}) the
convergence of our optimization procedure, under mild
conditions. A question that we have not tackled yet is to evaluate its
convergence rate. Figure~\ref{objdec} plots the evolution of the objective function on
the {\em sinc} dataset. We observe that the evolutions of the
objective function are impressively similar among the different
runs. This empirically tends to assert that it is relevant to
look for theoretical results on the convergence rate.

A question left for future work is the
impact of the random selection of the set of columns $\mathcal{S}$ on the  reached
solution.

\subsubsection{Zero-norm of $\bfmu$}
\begin{figure}[tb]
\centering
 \includegraphics[scale=0.18]{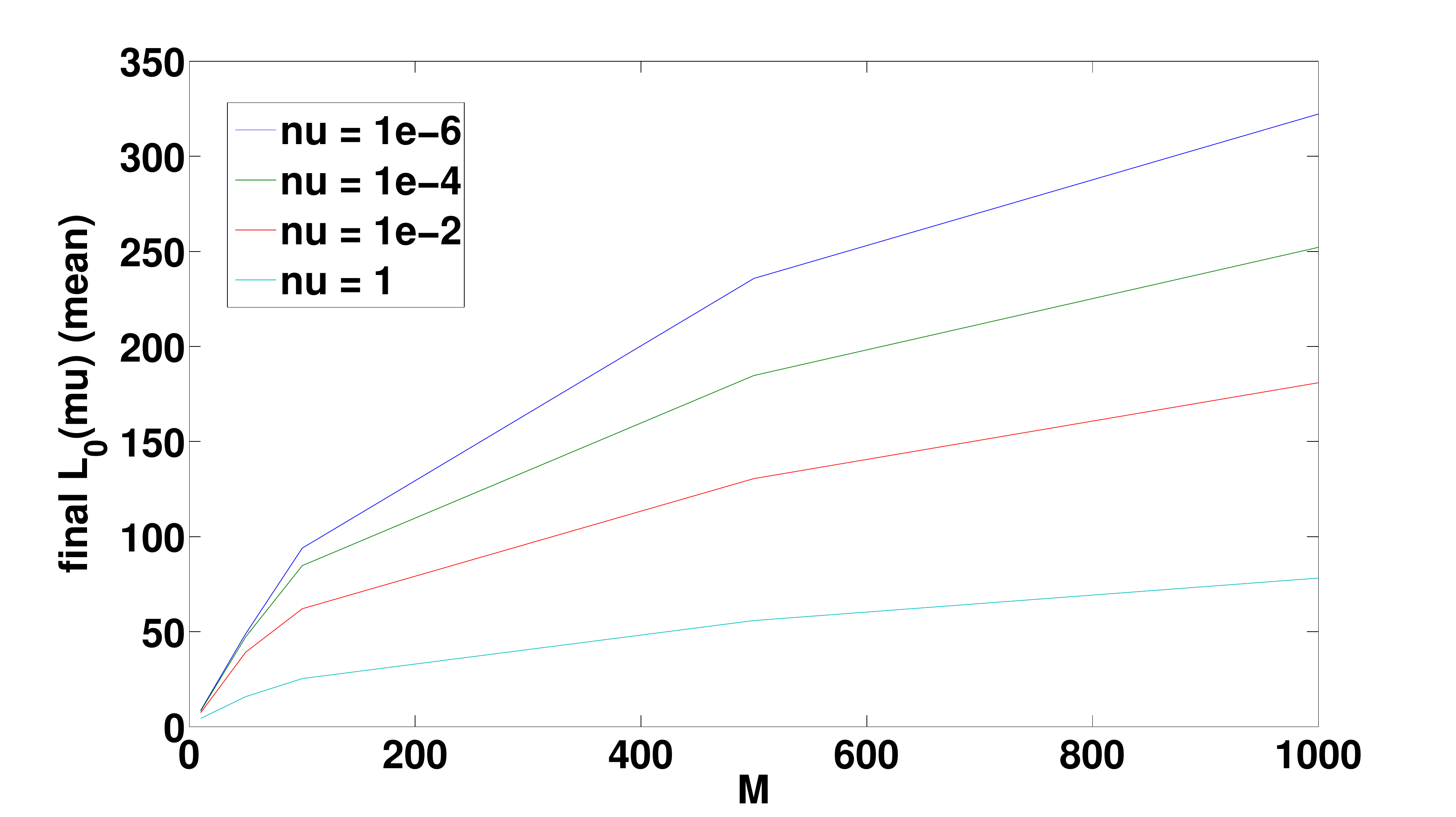}
\caption{Zero-norm of the optimal $\bfmu^*$ as a function of $M$ for different values of $\nu$ for the \emph{sinc} dataset (averaged on 20 runs).}
\label{n0final}
\end{figure}
\begin{figure}[tb]
 \centering
 \includegraphics[scale=0.18]{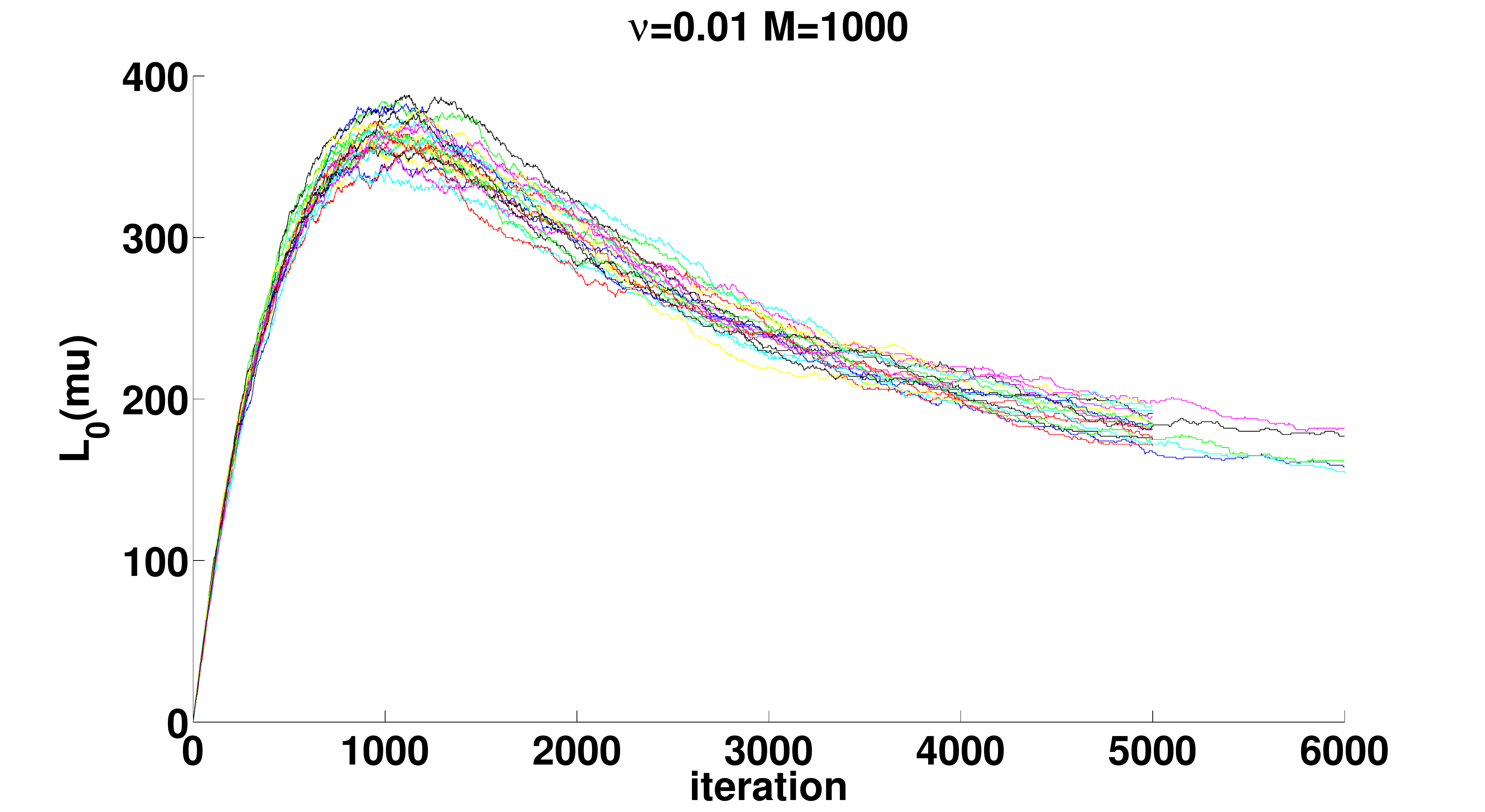}
\caption{Evolution of the zero-norm of $\bfmu$ ($m_0$) with the iterations
 for the \emph{sinc} dataset with $\nu = 0.01$, $M = 1000$ (20 runs).}
\label{evon0}
\end{figure}
As shown in Section~\ref{complexanalysis}, both memory usage and the complexity of the algorithm depend on $m_0$.
Thus, it is interesting to take a closer look at how this quantity evolves.
Figure~\ref{n0final} and~\ref{evon0} experimentally point out two things. On the one hand, the number of active components $m_0=\|\bfmu\|_0$ 
remains significantly smaller than $M$.
In other words, as long as the regularization parameter is well-chosen, we never have to store all of 
the $\bfc_m$ at the same time.
On the other hand, the solution $\bfmu^*$ is sparse and $\|\bfmu^*\|_0$ grows with $M$ and diminishes with $\nu$.
A theoretical study on the dependence of $\bfmu^*$ and $m_0$ in $M$ and $\nu$, left for future work, would be all the more interesting
since sparsity is the cornerstone on which the scalability of our algorithm depends.

\subsection{Comparison to other methods}
This section aims at giving a hint on how our method performs on regression tasks.
To do so, we compare the Mean Square Error (over the test set). 
In addition to our Stochastic Low-Rank Kernel Learning method (\emph{SLKL}), we solve 
the problem with the standard Kernel Ridge Regression method, using the $n$ training data 
(\emph{KRRn}) and using only $M$ training data (\emph{KRRM}).
We also evaluate the performance of the KRR method, using the kernel obtained with uniform 
weights on the $M$ rank-1 approximations selected for \emph{SLKL}~(\emph{Unif}).
The results are displayed in Table~\ref{compres}, where the bold font indicates the best low-rank method 
(\emph{KRRM, Unif} or \emph{SLKL}) for each experiment.

\begin{table*}[tb]
\centering
\caption{Mean square error with standard deviation measured on three regression tasks.}
\begin{footnotesize}
\begin{tabular}{ccccccccccccccccccc}
\toprule
  && \multicolumn{3}{c}{\emph{sinc}} && \multicolumn{3}{c}{\emph{boston}}&&\multicolumn{3}{c}{\emph{abalone}}\\
 $M$&&$256$&$512$&$1000$&&$128$&$256$&$350$&&$512$&$1024$&$3000$\\
\cmidrule{1-1}\cmidrule{3-5}\cmidrule{7-9}\cmidrule{11-13}
 \emph{KRRn}&&\multicolumn{3}{c}{$0.009 \pm0 9$}&&\multicolumn{3}{c}{$10.17 \pm0$}&&\multicolumn{3}{c}{$6.91 \pm0$}\\
\cmidrule{1-1}\cmidrule{3-5}\cmidrule{7-9}\cmidrule{11-13}
 \multirow{2}{*}{\emph{KRRM}}&&$0.0146$&$0.0124$&$\mathbf{0.0099}$&&$33.27$&$16.89$&$\mathbf{10.17}$&&$6.14$&$5.51$&$5.25$\\
 &&$\pm1e^{-3}$&$\pm7e^{-4}$&$\pm0$&&$\pm7.8$&$\pm3.27$&$\pm0$&&$\pm0.25$&$\pm0.09$&$\pm0$\\
\cmidrule{1-1}\cmidrule{3-5}\cmidrule{7-9}\cmidrule{11-13}
 \multirow{2}{*}{\emph{Unif}}&&$0.0124$&$0.0124$&$0.0124$&&$149.7$&$147.84$&$147.72$&&$10.04$&$9.96$&$9.99$\\
 &&$\pm1e^{-4}$&$\pm3e^{-5}$&$\pm0$&&$\pm5.57$&$\pm2.24$&$\pm0$&&$\pm0.17$&$\pm0.06$&$\pm0$\\
\cmidrule{1-1}\cmidrule{3-5}\cmidrule{7-9}\cmidrule{11-13}
 \multirow{2}{*}{\emph{SLKL}}&&$\mathbf{0.0106}$&$\mathbf{0.0103}$&$0.0104$&&$\mathbf{20.17}$&$\mathbf{13.1}$&$11.43$&&$\mathbf{5.04}$&$\mathbf{4.94}$&$\mathbf{4.95}$\\
 &&$\pm4e^{-4}$&$\pm2e^{-4}$&$\pm1e^{-4}$&&$\pm2.3$&$\pm0.87$&$\pm0.06$&&$\pm0.08$&$\pm0.03$&$\pm0.004$\\
$m_0$&&$83$&$108$&$139$&&$108$&$161$&$184$&&$159$&$191$&$253$\\
\bottomrule
\end{tabular}
\end{footnotesize}
\label{compres}
\end{table*}

Table \ref{compres} confirms that optimizing the weight vector $\bfmu$ is decisive as our 
results dramatically outperform those of \emph{Unif}.
As long as $M < n$, our method also outperforms \emph{KRRM}.
The explanation probably lies in the fact that our approximations keep information about similarities 
between the $M$ selected points and the $n-M$ others.
Furthermore, our method~\emph{SLKL} achieves comparable performances (or even better on \emph{abalone}) 
than \emph{KRRn}, while finding sparse solutions.
Compared to the approach from \cite{Smola00}, we seem to achieve lower test error on the \emph{boston} 
dataset even for $M=128$. On the \emph{abalone} dataset, this method outperforms ours for every M we tried.

Finally, we also compare the results we obtain on the \emph{USPS} dataset with the ones obtained 
in~\cite{Williams01Nystrom} (\emph{Nyst}). 
As it consists in a classification task, we actually perform a regression on the labels to adapt our 
method, which is known to be equivalent to solving Fisher Discriminant Analysis~\cite{Duda73}.
The performance achieved by \emph{Nyst} outperforms ours.
However, one may argue that the performance have a same order of magnitude and note that
the \emph{Nyst} approach focuses on the  classification task, while ours was
designed for regression.
\begin{table}[!h]
\centering
\caption{Number of errors and standard deviation on the test set (2007 examples) of the USPS dataset.}
\begin{footnotesize}
\begin{tabular*}{\columnwidth}{ccccccc}
\toprule 
$M$&&$64$&&$256$&&$1024$\\
\cmidrule{1-1}\cmidrule{3-3}\cmidrule{5-5}\cmidrule{7-7}
\emph{Nyst}
 &&$101.3 \pm22.9$&&$34.5 \pm3.0$&&$35.9 \pm2.0$\\
\cmidrule{1-1}\cmidrule{3-3}\cmidrule{5-5}\cmidrule{7-7}
\emph{SLKL}
 &&$76.3 \pm9.9$&&$47.6 \pm3.1$&&$41.5 \pm3.9$\\
 $m_0$&&$61$&&$210$&&$515$\\
\bottomrule
\end{tabular*}
\end{footnotesize}
 \label{compresSeeger}
\end{table}

\subsection{Large-scale dataset}
To assess the scalability of our method, we ran  experiments on the larger 
handwritten digits \emph{MNIST} dataset, whose training set is made of $60000$ examples. We used a Gaussian 
kernel computed over histograms of oriented gradients as in~\cite{Maji09}, in a ``one versus all'' setting. 
For $M\!=\!1000$, we obtained classification error rates around $2\%$ over the test set, which do not compete with 
state-of-the-art results but achieve reasonable performance, considering that we use only a small part of the data (cf. the size of $M$) and that our method was designed for regression.

Although our method overcomes memory usage issues for such large-scale problems, it still is 
computationally intensive. In fact, a large number of iterations is spent picking coordinates 
whose associated weight remains at $0$. Though those iterations do not induce any update, they do 
require computing the associated Gram matrix column (which is not stored as it does not 
weigh in the conic combination) as well as the derivatives of the objective function. The main focus 
of our future work is to avoid those computations, using e.g. techniques such as shrinkage~\cite{Hsieh08}.


\section{Conclusion}
\label{sec:conclusion}
We have presented an original kernel-based learning procedure for
regression. The main features of our contribution are the use of a
conical combination of data-based kernels and the derivation of a
stochastic convex optimization procedure, that acts coordinate-wise
and makes use of second-order information. We provide
theoretical convergence guarantees for this optimization procedure,
we depict the behavior of our learning procedure and illustrate its
effectiveness  through a number of numerical experiments carried out
on several benchmark datasets.

The present work naturally raises several questions. Among them, we may
pinpoint that of being able to
establish precise rate of convergence for the stochastic optimization
procedure and that of generalizing our approach to the use of several
kernels. Establishing data-dependent generalization
bounds taking advantage of either the one-norm constraint on $\bfmu$
or the size $M$ of the kernel combination is of primary importance to
us. The connection established between the one-norm hyperparameter $\nu$ and the ridge parameter
$\lambda$, in section~\ref{sec:analysis}, seems interesting and may be witnessed in~\cite{Rakoto08}.
Although not been mentioned so far, there might be
connections between our modeling strategy and
boosting/leveraging-based optimization
procedures.
Finally, we plan on generalizing our approach to other kernel
methods, noting that 
rank-1 update formulas as those proposed here can possibly be exhibited even for problems with no closed-form solution.

\subsection*{Acknowledgments}
This work is partially supported by the IST Program of the European Community, under the FP7 Pascal 2 Network of
Excellence (ICT-216886-NOE) and by the ANR project LAMPADA (ANR-09-EMER-007).
\appendix

\section{Matrix Inversion Formulas}
\label{sec:matrixinversion}
\begin{theorem}{\em (Woodbury matrix in\-version formu\-la~\cite{Woodbury50})}
\label{th:woodbury}
Let $n$ and $m$ be positive integers, $A\in\realset^{n\times n}$ and $C\in\realset^{m\times m}$ be non-singular matrices and let $U\in\realset^{n\times m}$ and $V\in\realset^{m\times n}$ be two matrices.
If $C^{-1} \!+\! VA^{-1}U$ is non-singular then so is $A\!+\!UCV$ and:
$$(A+UCV)^{-1}=A^{-1}-A^{-1}U(C^{-1} + VA^{-1}U)^{-1}VA^{-1}.$$
\end{theorem}

\begin{theorem}{\em (Matrix inversion with added column)}
\label{th:inversionadd}
Given $m$, integer and $M\in\realset^{(n+1)\times (n+1)}$  partitioned as:
 $$M=\begin{pmatrix}A & \bfb\\ \bfb^{\top} & c\end{pmatrix},\quad
\text{where } A\in\realset^{n\times n}, \bfb\in\realset^{n} \text{ and  } c\in\realset.$$
If $A$ is non-singular and $c-\bfb^{\top}A^{-1}\bfb\neq 0$, then $M$
is non-singular and the inverse of $M$ is given by
\begin{equation}
M^{-1}=\begin{pmatrix}A^{-1}+\frac{1}{k}A^{-1}\bfb\bfb^{\top}A^{-1}
  & -\frac{1}{k}A^{-1}\bfb\\
-\frac{1}{k}\bfb^{\top}A^{-1}&\frac{1}{k}\end{pmatrix},
\label{eq:inversionadd}
\end{equation}
where $k=c-\bfb^{\top}A^{-1}\bfb$.
\end{theorem}

\begin{small}
\bibliographystyle{icml2011}
\bibliography{nystromregression}
\end{small}
\end{document}